%% file: AISTATS25.tex
\documentclass[twoside]{article}

\usepackage[accepted]{aistats2025}
\usepackage{comment}
\usepackage{times}
\usepackage{soul}
\usepackage{url}
\usepackage[hidelinks]{hyperref}
\usepackage[utf8]{inputenc}
\usepackage{graphicx}
\usepackage{amsmath,amsfonts}
\usepackage{amsthm}
\usepackage{booktabs}
\usepackage{algorithm}
\usepackage{algorithmic}
\usepackage{cancel}
\usepackage{xcolor} 
\usepackage{natbib}

\newtheorem{theorem}{Theorem}

\newtheorem{proposition}[theorem]{Proposition}
\begin{document}

%

%

\twocolumn[

\aistatstitle{Bayesian Gaussian Process ODEs via Double Normalizing Flows}

\aistatsauthor{ Jian Xu \And Shian Du \And  Junmei Yang }
\aistatsaddress{ South China University\\ of Technology \And Tsinghua University \And South China University\\ of Technology} 

\aistatsauthor{ Xinghao Ding \And  Delu Zeng$^*$ \And John Paisley }
\aistatsaddress{Xiamen University\And South China University\\ of Technology \And Columbia University } ]

\runningauthor{Jian Xu, ~ Shian Du, ~  Junmei Yang, ~ Xinghao Ding, ~ Delu Zeng, ~ John Paisley}

\begin{abstract}
  Gaussian processes have been used to model the vector field of continuous dynamical systems, which are characterized by a probabilistic ordinary differential equation (GP-ODE). Bayesian inference for these models has been extensively studied and applied in tasks such as time series prediction. However, the use of standard GPs with basic kernels like squared exponential  kernels has been common in GP-ODE research, limiting the model's ability to represent complex scenarios. To address this limitation, we introduce normalizing flows to reparameterize the ODE vector field, resulting in a data-driven prior distribution, thereby increasing flexibility and expressive power. We develop a variational inference algorithm that utilizes analytically tractable probability density functions of normalizing flows. Additionally, we also apply normalizing flows to the posterior inference of GP-ODEs to resolve the issue of strong mean-field assumptions. By applying normalizing flows in these ways, our model improves accuracy and uncertainty estimates for Bayesian GP-ODEs. We validate the effectiveness of our approach on simulated dynamical systems and real-world human motion data, including time series prediction and missing data recovery tasks.
\end{abstract}

\section{INTRODUCTION}
Recent research has leveraged machine learning to model the behavior of dynamical systems governed by continuous-time differential equations 
\citep{hegde2022variational,chen2018neural,zhu2022numerical,chen2022forecasting,zhang2022metanode,heinonen2018learning}. Mathematically, such problems can be written as 
\begin{align}
\frac{d\boldsymbol{x}_t}{dt} &= \boldsymbol{f}(\boldsymbol{x}_t),  
\end{align}
where $\boldsymbol{x}_t \in \mathbb{R}^d$ represents a latent state, and $\boldsymbol{f}$ represents a vector field. One of the most promising approaches is the use of Gaussian process (GP) vector fields  to learn unknown nonlinear differential functions from state observations 
  through Bayesian inference, referred to as GP-ODEs \citep{hegde2022variational,heinonen2018learning}. In these dynamical system models, the vector fields are given Gaussian process priors for learning their unknown parameters.

Accuracy and uncertainty are crucial metrics in time series and dynamical systems. The  limitations of previous GP-ODE methods mainly falls into two areas. First, the learning of previous GP-ODEs has been constrained by the modeling capacity. By assuming the vector field $\boldsymbol{f}$ to be a standard Gaussian process with basic kernels, flexibility in modeling real world complex systems may be restricted. Second, inference for previously proposed GP-ODEs used a mean-field Gaussian posterior approximation for computational tractability, which may be too restrictive for complex dynamical systems. 

To address these limitations, we propose normalizing flows to reparameterize the vector field of ODEs, resulting in a more flexible and data-driven prior distribution. Normalizing flows have been proposed to transform a simple random variable into a more complex one by applying
a sequence of invertible and differentiable transformations \citep{kobyzev2020normalizing,papamakarios2021normalizing,zhang2021diffusion,maronas2021transforming}. For the complex system dynamics that a standard GP equipped with elementary kernel cannot model well, transformed GP priors are promising alternatives. Specifically, by applying the normalizing flow idea to transform the standard GP,  $\boldsymbol{f}$, we can get a more flexible and expressive random process as model definition. They simultaneously can be applied to posterior inference for GP-ODEs to learn non-Gaussian distributions during inference. Through these two applications of normalizing flows, we can improve accuracy and uncertainty estimates for Bayesian GP-ODEs. 

Specifically, in this paper we leverage the benefits of the analytically tractable probability density functions of normalizing flows to construct a variational inference algorithm that allows for simultaneous, precise, and reliable learning of the unknown continuous dynamics. The effectiveness of our approach is demonstrated on simulated dynamical systems and real-world human motion data, including tasks such as time series prediction and missing data recovery. The experimental results demonstrate that normalizing flows produce improved accuracy and uncertainty estimation over previous approaches. 


\section{METHOD}
In this section, we propose a Bayesian model based on normalizing flows to learn vector fields of ODEs with Gaussian processes. We review the background preliminaries before presenting our model and algorithm. We discuss related work in some detail in Section \ref{sec.related}.
\subsection{Gaussian process ODEs}
\label{Gaussian Process ODEs}
Gaussian Process ODEs are a class of nonparametric probabilistic models of continuous dynamic systems that evolve based on vector fields. The fundamental idea of a GP-ODE is to represent the ODE's unknown vector field as a Gaussian process, which can represent the smoothness, volatility, and randomness of the stochastic function at different input points. A numerical ODE solver can then be used to compute the evolution of the stochastic function over time.

The formulation of a GP-ODE model involves specifying a prior distribution over the latent function $\boldsymbol{f}$ and a likelihood function that establishes the relationship between the observed data and the latent function. More precisely, the model is defined by a zero-mean multidimensional Gaussian process prior \citep{moreno2018heterogeneous} over a dynamic process indexed by $t$ as follows
\begin{eqnarray}
\frac{d\boldsymbol{x}_t}{dt} &=& \boldsymbol{f}(\boldsymbol{x}_t),\nonumber \\
\boldsymbol{f}(\boldsymbol{x}_t) &\sim& \mathcal{GP}(\boldsymbol{0}, K(\boldsymbol{x}_t, \boldsymbol{x}_t')), \\
\boldsymbol{y}_t &\sim& \mathcal{N}(\boldsymbol{x}_t, R).\nonumber
\end{eqnarray}
Here, $\boldsymbol{x}_t \in \mathbb{R}^d$ represents the latent process, while $\boldsymbol{y}_t \in \mathbb{R}^d$ denotes the observation contaminated by noise with covariance $R$. The multidimensional function $\boldsymbol{f}(\boldsymbol{x}_t) \in \mathbb{R}^d$ represents the velocity field that characterizes the dynamics of the latent process over time. The GP prior on this function is defined by the kernel $K(\boldsymbol{x}_t, \boldsymbol{x}'_{t})$.

Inference for GP-ODE models entails computing the posterior distribution over the latent process $x_t$ and learn the function $\boldsymbol{f}$ given the observed data $\boldsymbol{y}_t$. This can be accomplished through numerical integration and variational inference. For instance, a popular approach by \cite{hegde2018deep} involves employing the Runge-Kutta method for numerical ODE solving to obtain an estimate of the latent function, followed by variational inference to approximate the posterior distribution.

One limitation of Gaussian process (GP) models is their sensitivity to the choice of kernel function, which can result in local underfitting when dealing with complex data structures. In general, even overly simple nonlinear models can introduce bias in parameter estimation and potential underfitting issues \citep{fortuin2022priors}. Furthermore, another limitation of the previous GP-ODE models is that the complexity of the data and the non-Gaussian distribution of the posterior often pose challenges for parameter inference. This is particularly challenging when using variational inference, as it can be difficult to find a suitable posterior distribution family. To address this challenge, in the next subsection we present our approach for enhancing the expressiveness of Gaussian Process ODE models using double normalizing flows.

\subsection{Double normalizing flows}
\subsubsection{ Flow \#1: Normalizing flow model prior}
 To increase the complexity of the GP-ODEs traditional prior distribution, we transform the Gaussian process prior through a normalizing flow, a family of flexible and invertible transformations that maps simple distributions to more complex ones \citep{rezende2015variational}.  
 
 To review, let $\boldsymbol{x}$ be a $d$-dimensional continuous random vector and $p(\boldsymbol{x})$ be its corresponding probability density distribution. Normalizing flows construct a desired and often multi-modal distribution by pushing $x$ through a series of transformations, $G_{K}(\boldsymbol{x})=\boldsymbol{g}_K\circ\cdots\boldsymbol{g}_2\circ \boldsymbol{g}_1 (\boldsymbol{x})$. By repeatedly applying change of variables rule,
the initial density moves through a sequence of invertible mappings, at the end of which we obtain a new distribution,
\begin{equation}
\ln p\left(G_{K}(\boldsymbol{x})\right)=\ln p\left(\boldsymbol{x}\right)-\sum_{k=1}^{K} \ln \left|\operatorname{det}\frac{\partial \boldsymbol{g}_{k}}{\partial G_{k-1}(\boldsymbol{x})}\right|    
\end{equation}
For example, a planar flow uses a simple single layer neural network structure for $\boldsymbol{g}_{k}$ defined as
 \begin{equation}
    \boldsymbol{g}(\boldsymbol{x})=\boldsymbol{x}+\boldsymbol{u} h\left(\mathbf{w}^{\top} \boldsymbol{x}+b\right),
 \end{equation}
where $\boldsymbol{u} \in \mathbb{R}^d, \mathbf{w} \in \mathbb{R}^d, b \in \mathbb{R}$ are free network parameters and $h(\cdot)$ is a non-linear smooth function \citep{papamakarios2021normalizing}. This transformation is invertible, and its Jacobian determinant can be computed efficiently as follows,
\begin{eqnarray}
\psi(\boldsymbol{x})&=&h^{\prime}\left(\mathbf{w}^{\top} \boldsymbol{x}+b\right) \mathbf{w},\nonumber \\
\operatorname{det}\left|\frac{\partial \boldsymbol{g}}{\partial \boldsymbol{x}}\right|&=&\left|\operatorname{det}\left(\mathbf{I}+\boldsymbol{u} \psi(\boldsymbol{x})^{\top}\right)\right|\\ 
&=&\left|1+\boldsymbol{u}^{\top} \psi(\boldsymbol{x})\right|.\nonumber
\end{eqnarray}

Applying this idea to a simple GP $\boldsymbol{f}(\cdot)$, we can transform it into a more flexible and expressive random process $G_K(\boldsymbol{f})(\cdot)$. Specifically, we form the vector field by first passing the state vector $\boldsymbol{x}_t$ through a multidimensional Gaussian process
\begin{equation}
\boldsymbol{f}(\boldsymbol{x}_t) \sim \mathcal{GP}(\boldsymbol{0}, K(\boldsymbol{x}_t,\boldsymbol{x}_t')).
\end{equation}
We then apply the series of transformations $G_{K}$ to the output of the Gaussian process,
\begin{equation}
 G_{K}(\boldsymbol{f}(\boldsymbol{x}_t)) = G_{K} \circ \boldsymbol{f}(\boldsymbol{x}_t) .
\end{equation}
The resulting function $G_{K} \circ \boldsymbol{f}$ can be used as the vector field in the ODE model since, by Kolmogorov’s consistency theorem \citep{oksendal2013stochastic}, it is a valid multivariate stochastic process in the same input space.
By compounding $K$ such transformations, we obtain a transformed prior distribution that is a highly flexible and expressive representation of the true vector field. 

As a result, we model the likelihood of the observed data $\boldsymbol{y}=\{\boldsymbol{y}_{t_1},\boldsymbol{y}_{t_2},\dots,\boldsymbol{y}_{t_N}\}$ as a noisy realization of the underlying hidden state $\boldsymbol{x}=\{\boldsymbol{x}_{t_1},\boldsymbol{x}_{t_2},\dots,\boldsymbol{x}_{t_N}\}$ through the function
\begin{equation}
\label{the model}
\boldsymbol{y}_t =\boldsymbol{x}_0+ \int_{0}^{t}  G_K\circ \boldsymbol{f}(\boldsymbol{x}_{\tau}) d\tau
 + \epsilon_t
\end{equation}
where $t_n \in [0,T]$, $G_K\circ \boldsymbol{f}$ is the vector field defined by the normalizing flow-transformed Gaussian process prior, and $\epsilon_t \sim \mathcal{N}(0,R)$ is additive Gaussian noise. We assume a $\mathcal{N}({0}, I)$ prior distribution over the initial variable $\boldsymbol{x}_0$.

In a traditional GP model, the log marginal likelihood is optimized in closed form. However, in the present framework, this quantity is intractable, and hence we resort to variational inference, which has advantages over maximum likelihood estimation \citep{blei2017variational,zhang2018advances}. By maximizing the variational objective, we can approximate the true posterior distribution over the latent variables. We discuss this inference algorithm next, again building our approximation on normalizing flows.

\subsubsection{Flow \#2: Normalizing flow posterior inference}
To learn the posterior distribution of $p(\boldsymbol{f}|\boldsymbol{y})$, we propose a variational inference method that again uses normalizing flows for distribution estimation. This learns a distribution $q(\boldsymbol{f})$ that approximates the true posterior distribution $p(\boldsymbol{f}|\boldsymbol{y})$ by minimizing the Kullback-Leibler (KL) divergence between $q(\boldsymbol{f})$ and $p(\boldsymbol{f}|\boldsymbol{y})$. 


To speed up the computation and relax the dependence on the entire data, we use the variational sparse GP (SVGP) \citep{titsias2009variational,hensman2015scalable,mao2020multiview,liu2020gaussian,xu2024sparse}.
SVGP introduces a set of inducing points $\boldsymbol{z}=\{\boldsymbol{z}_1,\boldsymbol{z}_2,...,\boldsymbol{z}_M\}$ and inducing variables $\boldsymbol{u}=\{\boldsymbol{u}_1,\boldsymbol{u}_2,...,\boldsymbol{u}_M\}$ with a prior distribution $p(\boldsymbol{u})=\mathcal{N}(0,K(\boldsymbol{z},\boldsymbol{z}'))$. The inducing variables are used to sparsely represent the Gaussian process vector field $\boldsymbol{f}$, which reduces the complexity of the Gaussian
process inference by introducing the approximation
\begin{eqnarray}\label{sparse inducing
points}
p(\boldsymbol{f}_t|\boldsymbol{u}) &\hspace{-2pt}=\hspace{-2pt}& \mathcal{N} (\mu_{f_t},\Sigma_{f_t})\\
\mu_{f_t} & \hspace{-2pt}=\hspace{-2pt} & K(\boldsymbol{x}_t,\boldsymbol{z})K^{-1}(\boldsymbol{z},\boldsymbol{z}')\boldsymbol{u}\nonumber\\
\Sigma_{f_t} & \hspace{-2pt}=\hspace{-2pt} & K(\boldsymbol{x}_t,\boldsymbol{x}_t')-K(\boldsymbol{x}_t,\boldsymbol{z})K^{-1}(\boldsymbol{z},\boldsymbol{z}')K(\boldsymbol{z},\boldsymbol{x}_t)\nonumber
\end{eqnarray}
 where covariance matrix $K(\boldsymbol{z},\boldsymbol{z}')$ is a kernel function calculated between the inducing points $\boldsymbol{z}$. Like previous work with this approximation \citep{titsias2009variational,hensman2015scalable}, we factor the joint approximate posterior distribution for the Gaussian process as
 \begin{equation}
q(\boldsymbol{f},\boldsymbol{u})=q(\boldsymbol{f}|\boldsymbol{u})q(\boldsymbol{u})=p(\boldsymbol{f}|\boldsymbol{u})q(\boldsymbol{u}).
 \end{equation}
We consider the ODE model proposed in Equation (\ref{the model}) augmented by sparse inducing
points presented in Equation (\ref{sparse inducing points}). In Bayesian machine learning, the model evidence function $p(\boldsymbol{y})$ is used for inference by establishing an upper bound on the variational objection function, where the difference is the KL divergence between an approximating distribution and the true model posterior. For our model, this is
\begin{equation}
\label{log likelihood}
    \log p(\boldsymbol{y})\ge \mathcal{L}=\mathbb{E} _{q(\boldsymbol{x}_0,\boldsymbol{f},\boldsymbol{u})}\left[ \log \frac{p(\boldsymbol{f},\boldsymbol{u},\boldsymbol{x}_0,\boldsymbol{y})}{q(\boldsymbol{x}_0,\boldsymbol{f},\boldsymbol{u})} \right] 
\end{equation}
where $q( \boldsymbol{u} ), q( \boldsymbol{x}_0)$ are the variational posteriors of $\boldsymbol{u}$ and $\boldsymbol{x}_0$, and $\lambda$ are other model parameters such as Gaussian kernel length scale. We use this ELBO function as the objective to optimize the variational distribution $q(\boldsymbol{u})$. A standard calculation shows that Equation (\ref{log likelihood}) can be mathematically simplified to
\begin{equation}
\label{elbo}
\begin{aligned}
\mathcal{L} =& ~~ {{\mathbb{E} _{q(\boldsymbol{f})q(\boldsymbol{x}_0)}\big[\log p(\boldsymbol{y}|G_K\circ \boldsymbol{f},\boldsymbol{x}_0)\big]}}
\\
\\&-{{\mathrm{KL}\big[q(\boldsymbol{u})||p(\boldsymbol{u})\big]}}-{{\mathrm{KL}\big[q(\boldsymbol{x}_0)||p(\boldsymbol{x}_0)\big]}},
\end{aligned}
\end{equation}
where $\mathbb{E}_{q(\boldsymbol{f})q(\boldsymbol{x}_0)}[\log p(\boldsymbol{y}|G_K\circ\boldsymbol{f},\boldsymbol{x}_0]$ is the expected log-likelihood of the data under the variational distributions $q(\boldsymbol{f})$ and $q(\boldsymbol{x}_0)$. In the context of inducing points, we represent $q(\boldsymbol{f})=\int{q\left( \boldsymbol{f},\boldsymbol{u} \right) d\boldsymbol{u}}
$ with samples from the variational posterior using Matheron's Rule \citep{wilson2020efficiently}. 
The likelihood function can be written as
\begin{equation}
\label{p(y|u)}
p(\boldsymbol{y}|G_K\circ\boldsymbol{f},\boldsymbol{x}_0)=\mathcal{N}\left(\boldsymbol{x}_0+\textstyle\int_{0}^{t} G_K\circ \boldsymbol{f}(\boldsymbol{x}_{\tau}) d\tau,R\right).
\end{equation}
We optimize the ELBO using stochastic gradient descent, where the gradients are computed using Monte Carlo sampling. The integral in Equation (\ref{p(y|u)}) can be solved using an ODE solver. 

It remains to define the $q(\boldsymbol{u})$ distribution to be used for the sampled gradients. As discussed in Section \ref{Gaussian Process ODEs}, in order to obtain a more accurate estimation of the model's uncertainty, we parameterize $q(\boldsymbol{u})$ used in Equation (\ref{elbo}) with another  normalizing flow. For example, a planar flow gives
\begin{align}
\label{U posterior}
 \begin{aligned} 
\boldsymbol{u}&=\phi \left( \boldsymbol{v} \right)\\ 
q(\boldsymbol{u})&=\pi (\boldsymbol{v}) \Big| \det \left( \frac{d\phi}{d\boldsymbol{v}} \right) \Big|^{-1}\\
&=\pi (\boldsymbol{v})\left|1+\boldsymbol{u}^{\top} \psi(\boldsymbol{v})\right|^{-1}
\end{aligned}   
\end{align}
where $\boldsymbol{v}$ is a new random
variable with a standard normal distribution $\pi(\cdot)=\mathcal{N}(\boldsymbol{0},\boldsymbol{I})$. Then using Equation (\ref{U posterior}), the KL divergence in Equation (\ref{elbo}) can be re-written as
\begin{align}
\label{KL term}
    \begin{aligned}
 &\mathrm{KL(}q(\boldsymbol{u})\| p(\boldsymbol{u})) =~~~~
 \\
 &~~~~~~\mathbb{E} _{\pi (\boldsymbol{v})}[ \log \pi (\boldsymbol{v})-\log \Big| \det \left( \frac{d\phi}{d\boldsymbol{v}} \right) \Big|-\log p(\phi (\boldsymbol{v}))]\\
 &~~~~~~=-\mathbb{E} _{\pi (\boldsymbol{v})}\log \left| 1+\boldsymbol{u}^{\top}\psi (\boldsymbol{v}) \right|+\frac{1}{2}\log |K\left( \boldsymbol{z},\boldsymbol{z}' \right) |
 \\& ~~~~~~~~~~~~+\frac{1}{2}\mathbb{E} _{\pi (\boldsymbol{v})}\boldsymbol{v}^TK^{-1}\left( \boldsymbol{z},\boldsymbol{z}' \right) \boldsymbol{v}+\mathrm{const.}
\end{aligned}
\end{align}
We can still utilize Monte Carlo sampling to obtain an unbiased estimate of the KL term from the Equation (\ref{KL term}). By combining Equations (\ref{elbo}) and (\ref{KL term}), we construct the complete objective function for double normalizing flows.

\section{RELATED WORKS}\label{sec.related}
Before presenting experimental results, this section provides an overview of the existing literature on  ODEs and normalizing flows, as well as their relationships. 
\paragraph{Non-parametric ODE models.}
Previous research has explored combining GP priors with ODEs for model specification \citep{heinonen2018learning,hegde2022variational}. These approaches offer flexible and non-parametric ways to model complex dynamics in various applications. Specifically, \citet{heinonen2018learning} proposes learning unknown, non-linear differential functions from state observations using Gaussian process vector fields, with the assumption that the vector field is a random function drawn from a Gaussian process. They provide a method to parameterize the ODE model with inducing points and use adjoint sensitivity equations to efficiently compute gradients of the system. Building on this, \citet{hegde2022variational} proposes a Bayesian perspective for posterior inference of the model and a probabilistic shooting augmentation to enable efficient inference for long trajectories. 
\paragraph{Normalizing flows and  dynamic modeling.}
Normalizing flows, as a class of generative models, have been shown to be flexible and expressive, allowing for complex and high-dimensional data modeling. 
Several studies have proposed using normalizing flows to enhance dynamic modeling. For example, \citet{deng2020modeling,deng2021continuous} proposed a type of normalizing flow driven by a differential deformation of the Wiener process, which inherits many appealing properties of its base process. \citet{zhi2022learning} proposed learning an ODE of interest from data by viewing its dynamics as a vector field related to another base vector field by an invertible neural network, from a perspective of differential geometry.

\paragraph{Normalizing flows and  GP models.} \citet{maronas2021transforming}  introduced the Transformed Gaussian Processes (TGP) framework, which combines Gaussian Processes (GPs) as flexible and non-parametric function priors with a parametric invertible transformation. The primary objective is to expand the range of priors and incorporate interpretable prior knowledge, including constraints on boundedness. In a subsequent work, \citet{maronas2023efficient} applied the TGP framework to create non-stationary stochastic processes that are inherently dependent. The TGP approach demonstrates particular suitability for addressing multi-class problems involving a substantial number of classes. 
While our paper also incorporates the idea of TGPs, our approach differs in the sense that our baseline is a continuous dynamical system. Moreover, we consider the complexity of the posterior distribution and propose a novel approach called double normalizing flows, which further improves the model's uncertainty estimation in time series prediction.


\begin{table}[b]
    \centering
          {\begin{tabular}{c|c}
        
        \hline\hline
        Method & Paper   \\
        \hline
        Bayesian NeuralODE & \citet{dandekar2020bayesian}  \\
        NeuralODE & \citet{chen2018neural}  \\
        npODE & \citet{heinonen2018learning}  \\
        GP-ODE & \citet{hegde2022variational} \\
        ODE2VAE & \citet{yildiz2019ode2vae} \\
        Latent SDE & \citet{solin2021scalable} \\
        \hline\hline
      \end{tabular}
      }
    
    \caption{A list of the models compared with.}
    \label{algorithmlist}
    
\end{table}

\section{EXPERIMENTS}
In this section, we empirically evaluate our approach compared with state-of-the-art methods (see Table \ref{algorithmlist}) including Bayesian NeuralODE \citep{dandekar2020bayesian}, NeuralODE \citep{chen2018neural}, npODE \citep{heinonen2018learning}, GP-ODE \citep{hegde2022variational}, ODE2VAE \citep{yildiz2019ode2vae} and Latent SDE \citep{solin2021scalable} in scenarios that encompass various challenges, including time series prediction, missing observations, and complex long trajectories.

\begin{figure*}[th!] 
  \centering       
    \includegraphics[trim={10mm 0mm 0mm 0mm},clip,width=.98\textwidth]{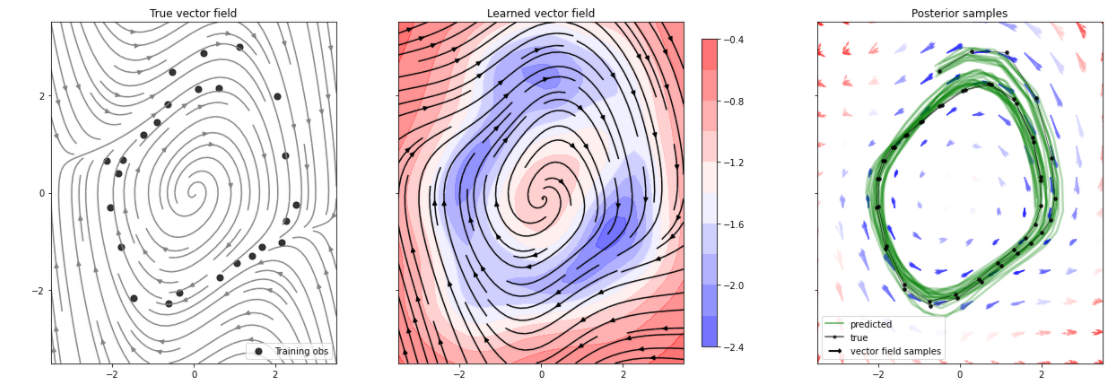} 
  \caption{Left-to-right: The true vector field, the vector field learned for this VDP dynamical system, the posterior trajectories sampled from our GP-DNF method. The log variance gradually decreases from red to blue.} 
  \label{fig1}  
\end{figure*}

\subsection{Predicting simulated dynamical systems}
We initially demonstrate our proposed method by incorporating double normalizing flows  (GP-DNFs) into the vector field of a two-dimensional Van der Pol (VDP) system \citep{kanamaru2007van}. This system is defined by the differential equations,
\begin{eqnarray}\label{eq.vdp}
\dot{x}_{1}&=&x_{2}, \\
\dot{x}_{2}&=&-x_{1}+0.5 x_{2}\left(1-x_{1}^{2}\right) .\nonumber
\end{eqnarray}

We generated a trajectories of length 50 by simulating the true system dynamics from the initial state $(\boldsymbol{x}_1(0), \boldsymbol{x}_2(0))=\left(-1.5,2.5\right)$, and added Gaussian noise with $\sigma^2=0.05$ to create the training data. We then explored two scenarios with different training time intervals, $t \in \left[0, 7\right]$ and forecasting intervals $t \in \left[7, 14 \right]$ respectively, using a regularly sampled time grid (Task 1) and an irregular grid (Task 2) with uniform random sampling of time points. We calculate mean squared error (MSE) and mean negative log likelihood (MNLL) as performance metrics.

\begin{figure*}[t!] 
  \centering       
    \includegraphics[width=.98\textwidth]{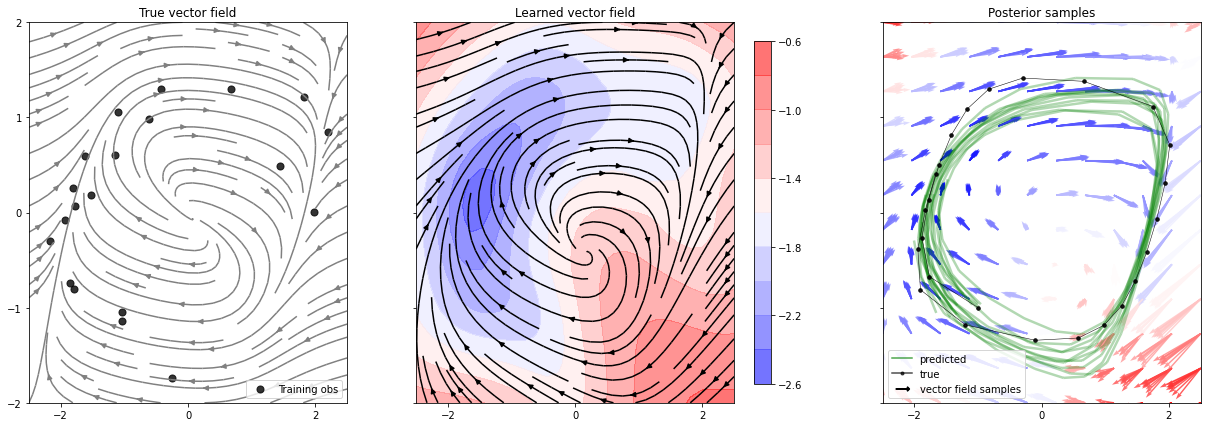} 
  \caption{Left-to-right: The true vector field, the vector field learned for this FHN dynamical system, the posterior trajectories sampled from our GP-DNF method. The log variance gradually decreases from red to blue.} 
  \label{fig2}  
\end{figure*}

\input{Table1}

\input{Table2}

\input{Table3}
\input{Table4}

From Table \ref{table1}, it can be observed that our proposed method exhibits lower MSE and MNLL values compared to the other five methods. Figure \ref{fig1}  shows the  vector field learned from the proposed method and sampled trajectories from the posterior of this vector field. These illustrative synthetic experimental results emphasize the effectiveness of double normalizing flows for time series prediction.

\subsection{Simulated dynamical systems with missing data}

We continue to evaluate the impact of double normalizing flows on the performance of time series recovery using the FitzHugh-Nagumo (FHN) oscillator \citep{aqil2012synchronization} defined by the differential equations
\begin{eqnarray}\label{eq.fnh}
\dot{x}_{1}&=&3\left(x_{1}-x_{1}^{3} / 3+x_{2}\right), \\
\dot{x}_{2}&=&\frac{1}{3}\left(0.2-3 x_{1}-0.2 x_{2}\right).\nonumber
\end{eqnarray}
We first  generate a training sequence with 25 regularly-sampled time points from $t \in \left[0, 5.0\right]$ and add Gaussian noise with $\sigma^2 = 0.025$. Then, we remove all observations in the quadrant where $\boldsymbol{x}_1 > 0$ and $\boldsymbol{x}_2 < 0$, and assess the accuracy of our model in this region. Figure \ref{fig2} shows the learned vector field and sampled trajectories from the posterior of this vector field. Based on the quantitative results in Table \ref{table2}, our double normalizing flow model again demonstrates improvement in performance compared to the previous methods on the missing data task.

\subsection{Complex trajectories of real-world datasets}
As a real world problem, we apply our double normalizing flow method to learn the dynamics of human motion using noisy experimental data from the CMU MoCap database for subjects 09, 35, and 39.\footnote{http://mocap.cs.cmu.edu/} The dataset comprises 50 sensor readings from various body parts during walking or running. We preprocess the data by centering it and splitting it into train, test, and validation sequences. To reduce dimensionality, we use Principal Component Analysis (PCA) to project the 50-dimensional data into a 5D latent space for dynamical learning. We evaluate the model's performance on both long and short sequences and employ the multiple shooting optimization framework \citep{osborne1969shooting,bock1984multiple,bock1983recent,biegler1986nonlinear}  recently proposed for probabilistic GP-ODEs by \citet{hegde2022variational} with the number of shooting segments matching the observation segments in the dataset.

In Figures \ref{3dtraj} and \ref{fig3}, we show the sampling of inducing points and particle trajectories in the latent space using a 3D model of the latent space. To compute the data likelihood, we project the latent dynamics back to the original data space by inverting PCA. We also evaluate the predictive performance on unseen test sequences  of our approach against other state-of-the-art methods on MNLL and MSE metrics in Tables \ref{table3} and \ref{table4}, respectively. The results demonstrate that our approach can achieve better accuracy on complex datasets.

\begin{figure}[h!]
    \centering
    \includegraphics[width=\columnwidth]{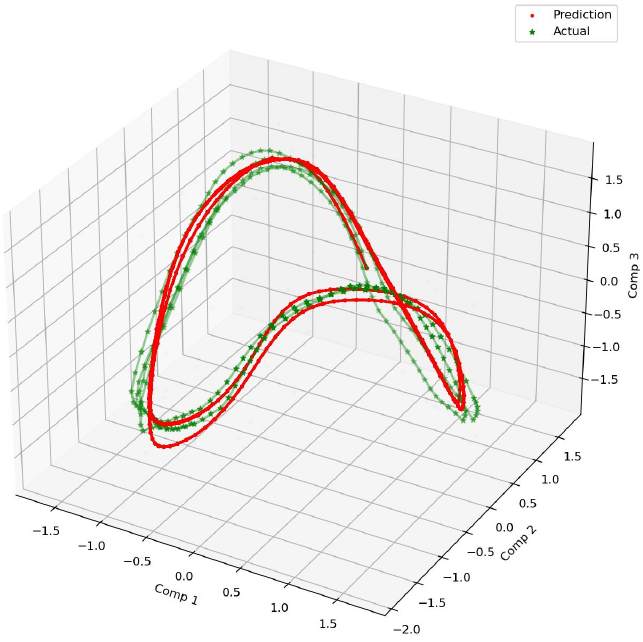}
    \caption{A 3D representation of the learned vector field in the PCA latent space for the mcap39 data set using our GP-DNF model. The red line denotes the predicted values, while the green line indicates the actual values. In Figure \ref{fig3} we show the sampled inducing variable vector fields learned along 2D slices of this image.}
    \label{3dtraj}
\end{figure}

\input{mocap}

\begin{figure*}[ht!] 
  \centering       
    \includegraphics[trim={4mm  0 10mm 0},clip,width=1.0\textwidth,height=0.35\linewidth]{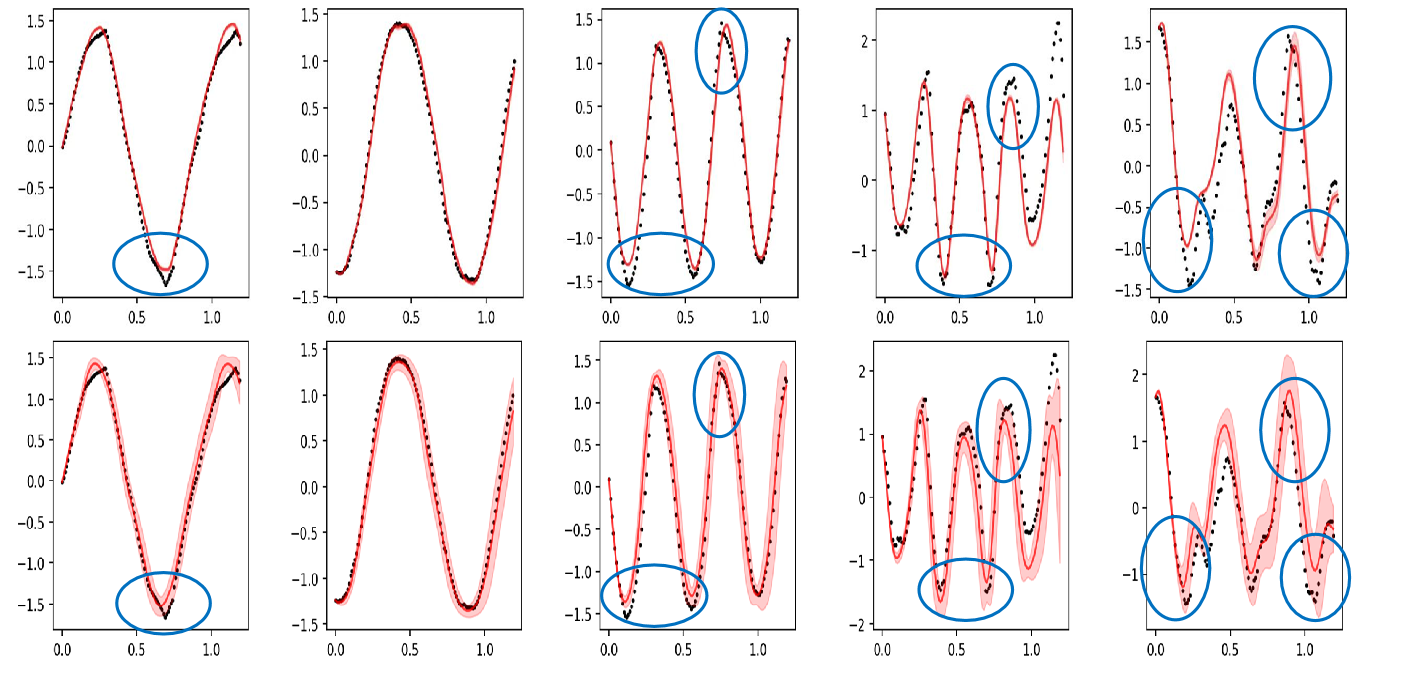} 
  \caption{The top row shows the results of GP-ODE, while the bottom row shows the results of our GP-DNF method. The black dashed line represents the true testing dataset, with time on the x-axis and data values on the y-axis. The deep red solid line represents the predicted mean value, and the light red shaded area represents the 95\% credible interval used to assess uncertainty. We have also highlighted areas of difference between the two methods with blue circles. Our method has a coverage probability of 0.81, while GP-ODE has a coverage probability of 0.60.} 
  \label{fig4}  
\end{figure*}

In addition, to demonstrate the importance of using double normalization flow for our performance, we also report the coverage probability. The coverage probability measures the proportion of true values within the 95\% confidence interval, thereby evaluating the uncertainty estimation in our time series prediction task on the MoCap09 test set. By visualizing the coverage probability in Figure \ref{fig4}, we  show prediction results indicating where our GP-DNF method outperforms GP-ODE, circled in blue. This indicates that GP-DNF is capable of providing better uncertainty estimation in addition to better MSE and MNLL.

 \begin{table}[hb!]
    \centering
    \resizebox{1\width}{!}{
        \begin{tabular}{|l|c|c|c|}
            \hline
            Dataset & MoCap09 & MoCap35 & MoCap39 \\
            \hline
            \rule{0pt}{0.5ex} 
            GP-ODE & $170s$ & $412s$ & $159s$ \\
            \hline
            \rule{0pt}{0.5ex} 
            GP-DNF (ours) & $210s$ & $435s$ & $187s$ \\
            \hline
        \end{tabular}
    }
    \caption{Comparison of the time required in seconds to learn GP-ODE and GP-DNF after 500 iterations.}
    \label{table5}
\end{table}

\subsection{Runtime analysis}
 Although our GP-DNF method demonstrates an ability to improve the modeling and posterior inference compared to the previous GP-ODE approach, the introduction of new parameters for learning and inference does lead to additional runtime. We provide a comparison between our proposed model and the baseline model on the MoCap dataset. Table \ref{table5} shows the time required to learn GP-DNF and GP-ODE after running 500 iterations using the NVIDIA A100 platform. It can be observed that as the dataset size increases, the primary computational resources are allocated towards handling other components, rather than computing the normalizing flows. As a result, our model demonstrates higher adaptability to changes in dataset size, and is not significantly affected by its increase any more than GP-ODE.

\subsection{Additional details on experimental setup}
For GP-DNF, we used 16 inducing points in the VDP and FHN experiments, and 100 inducing points for the MoCap experiments. This choice was made because the MoCap dataset has more sequence samples, and the selection of inducing points follows a trade-off between accuracy and efficiency. We assume Gaussian observation likelihood, and infer the observation noise
parameter from the training data. All the experiments use squared exponential kernel  along with 256 Fourier basis functions for weight-space  GP sampling methods according to Matheron's Rule \citep{wilson2020efficiently}. Along with the variational parameters, normalizing flows parameters, kernel length scales, signal variance, noise scale, and inducing locations are jointly optimized against the model ELBO while training. In the human motion dynamics experiments, we perform the same preprocessing on the data as in previous works, and made separate predictions for long-term and short-term time series. We use multiple shooting method \citep{hegde2022variational} and considered the number of shooting segments to be the same as the number of observation
segments in the dataset.

We quantify model fit by computing the mean negative log-likelihood (MNLL) of the data points under the predictive distribution and the mean squared error (MSE) with respect to the ground truth. Additionally, we report coverage probability, which measures the proportion of times the true value falls within the 95\% credible interval, to evaluate uncertainty estimation in the human motion dynamics experiments. All experiments are conducted on a single NVIDIA A100 GPU and are repeated 10 times with random initialization. Means and standard errors are reported across multiple runs. While different configurations may influence the results, we find that the model is not highly sensitive to hyperparameter choices. To guide hyperparameter selection, we employed cross-validation and observed consistent performance across various settings.

\section{CONCLUSION}
This paper proposes the use of double normalizing flows (GP-DNF) for the GP ordinary differential equation (GP-ODE) model. The method leverages the flexibility and analytical nature of normalizing flows to mathematically characterize vector fields and particle trajectory behaviors. Our results compare favorably with other methods, but there remain some limitations to this method for future work, such as the selection of hyperparameters and the need to carefully apply learning techniques such as early stopping and regularization to prevent overfitting, as well as numerical stability concerns related to the KL regularization term.

\appendix
\section{APPENDIX}
  \subsection{Prior Normalizing Flows}
  \begin{proposition}
  Given a prior normalizing flow, denoted as $G_K(\cdot)$, and a multidimensional Gaussian process $f(\cdot)$ with an input space $\mathcal{X}$, then $G_K \circ f (\cdot)$ is a valid multivariate stochastic process in the same input space.
  \end{proposition}

  \begin{proof}
  To establish the validity of $G_K \circ f (\cdot)$ as a multivariate stochastic process, we first note that $G_K(\cdot)$ is both invertible and differentiable, ensuring the well-definedness of the push-forward measure. Then, the collection of finite-dimensional distributions $\left\{G_K\circ f (x_1),G_K\circ f (x_2),...,G_K\circ f (x_T)\right\}$, where $T \in \mathbb{N}$, from $G_K\circ f (x)$, satisfy the consistency conditions of the well-known Kolmogorov's consistency theorem. This is due to the coordinate-wise mapping nature of $G_K(\cdot)$, which completes the proof.
  \end{proof}

 \subsection{Derivation of Variational Lower Bound}
 \begin{proposition}
 We denote $G_K(\cdot)=\boldsymbol{g}_K\circ ...\boldsymbol{g}_2\circ \boldsymbol{g}_1 (\cdot)$, $
 \tilde{f}\left( \cdot \right) =G_K\circ f\left( \cdot \right)
 $, and for the inducing variables, we adopt the same operation, i.e. $
 \tilde{\boldsymbol{u}}\left( \cdot \right) =G_K\circ \boldsymbol{u}\left( \cdot \right)
 $. The transformed prior after introducing the sparse inducing points is given by:
 \begin{equation}
    p\left( \tilde{f},\boldsymbol{\tilde{u}} \right) =p\left( f|\boldsymbol{u} \right) J_f\cdot p\left( \boldsymbol{u} \right) J_{\boldsymbol{u}}
\end{equation}
 where $
 J_f=\prod_k{\left| \frac{\partial \boldsymbol{g}_{k+1}}{\partial G_k\circ f} \right|^{-1}}
 $ and $J_{\boldsymbol{u}}=\prod_k{\left|  \frac{\partial \boldsymbol{g}_{k+1}}{\partial G_k\circ \boldsymbol{u}} \right|^{-1}}$
 \end{proposition}

\begin{proof}
    According to the variable transformation formula, we can obtain:
\begin{equation}
\begin{aligned}
p(\tilde{f},\tilde{\boldsymbol{u}})&=p(\boldsymbol{f},\boldsymbol{u})\prod_{j=0}^{J-1}\left|\left( \begin{matrix}
	\frac{\partial \boldsymbol{g}_k}{\partial G_{k-1}\circ f}&		     0\\
	     0&		\frac{\partial \boldsymbol{g}_k}{\partial G_{k-1}\circ {\boldsymbol{u}}}\\
\end{matrix} 
\right)\right|^{-1}\\
&=p\left( \boldsymbol{f}|{\boldsymbol{u}} \right) J_f\cdot p\left( {\boldsymbol{u}} \right) J_{\boldsymbol{u}}
\end{aligned}
\end{equation}
  The Jacobian matrix has a block-diagonal structure due to the fact that we assume $G_K(\cdot)$ is a coordinate-wise map.
\end{proof}
  Thus, we can obtain a joint model probability density function:
\begin{equation}
\begin{aligned}
    &p(\boldsymbol{\tilde{f}},\tilde{{\boldsymbol{u}}},\boldsymbol{x}(0),Y)=p( Y|\boldsymbol{\tilde{f}},\boldsymbol{x}( 0 ) ) p( \boldsymbol{\tilde{f}},\tilde{{\boldsymbol{u}}} ) p( \boldsymbol{x}\left( 0 \right) ) 
\\
&=p( Y|\boldsymbol{\tilde{f}},\boldsymbol{x}\left( 0 \right) ) p( \boldsymbol{f}|{\boldsymbol{u}} ) J_f\cdot p( {\boldsymbol{u}} ) J_{\boldsymbol{u}}\cdot p( \boldsymbol{x}\left( 0 \right) ) 
\end{aligned}
\end{equation}
Similarly, using the mean field approximation, we can derive the joint probability density function of the variational distribution,
\begin{equation}
    q(\boldsymbol{x}\left( 0 \right) ,\boldsymbol{\tilde{f}},\tilde{U})=p\left( f|{\boldsymbol{u}} \right) J_f\cdot q\left( {\boldsymbol{u}} \right) J_{\boldsymbol{u}}\cdot q\left( \boldsymbol{x}\left( 0 \right) \right) 
\end{equation}
By doing so, we obtained a explicit expression for the variational lower bound.
\begin{equation}
\begin{aligned}
    &\log p(Y\mid \boldsymbol{\theta })\ge \mathrm{ELBO}
=\mathbb{E} _{q(x\left( 0 \right) ,\boldsymbol{\tilde{f}},\tilde{{\boldsymbol{u}}})}\left[ \log \frac{p(\boldsymbol{\tilde{f}},\tilde{{\boldsymbol{u}}},X,Y)}{q(\boldsymbol{x}\left( 0 \right) ,\boldsymbol{\tilde{f}},\tilde{{\boldsymbol{u}}})} \right] 
\\
&=\mathbb{E} _{q(x\left( 0 \right) ,\boldsymbol{\tilde{f}})}\left[ \log p\left( Y|\boldsymbol{\tilde{f}},\boldsymbol{x}\left( 0 \right) \right) \right] 
\\
&-\mathbb{E} _{q(x\left( 0 \right) ,\boldsymbol{\tilde{f}},\tilde{{\boldsymbol{u}}})}\left[ \log \frac{{p\left( f|{\boldsymbol{u}} \right) J_f}\cdot q\left( {\boldsymbol{u}} \right) {J_{\boldsymbol{u}}}\cdot q\left( \boldsymbol{x}\left( 0 \right) \right)}{{p\left( f|{\boldsymbol{u}} \right) J_f}\cdot p\left( {\boldsymbol{u}} \right) {J_{\boldsymbol{u}}}\cdot p\left( \boldsymbol{x}\left( 0 \right) \right)} \right] 
\\&
=\mathbb{E} _{q(\boldsymbol{x}\left( 0 \right) ,\boldsymbol{f})}\left[ \log p\left( Y|G_K\circ \boldsymbol{f},\boldsymbol{x}\left( 0 \right) \right) \right] \\&-\mathrm{KL}\left( q\left( \boldsymbol{x}\left( 0 \right) \right) |p\left( \boldsymbol{x}\left( 0 \right) \right) \right) -\mathrm{KL}\left( q\left( {\boldsymbol{u}} \right) |p\left( {\boldsymbol{u}} \right) \right) \end{aligned}
\end{equation}

The last equation is formally known as probability under change of measure \citep{rezende2015variational}.



\section*{Acknowledgments}
The work is supported by the Fundamental Research Program of Guangdong, China, under Grant 2023A1515011281, and in part by the National Natural Science Foundation of China under Grant 61571005.

\bibliographystyle{apalike}
\bibliography{ijcai24}

 \section*{Checklist}

 \begin{enumerate}

  \item For all models and algorithms presented, check if you include:
   \begin{enumerate}
     \item A clear description of the mathematical setting, assumptions, algorithm, and/or model. [Yes]
    \item An analysis of the properties and complexity (time, space, sample size) of any algorithm. [Yes]
     \item (Optional) Anonymized source code, with specification of all dependencies, including external libraries. [Yes]
   \end{enumerate}

   \item For any theoretical claim, check if you include:
   \begin{enumerate}
     \item Statements of the full set of assumptions of all theoretical results. [Yes]
     \item Complete proofs of all theoretical results. [Yes]
     \item Clear explanations of any assumptions. [Yes]     
   \end{enumerate}

   \item For all figures and tables that present empirical results, check if you include:
   \begin{enumerate}
     \item The code, data, and instructions needed to reproduce the main experimental results (either in the supplemental material or as a URL). [Yes]
     \item All the training details (e.g., data splits, hyperparameters, how they were chosen). [Yes]
           \item A clear definition of the specific measure or statistics and error bars (e.g., with respect to the random seed after running experiments multiple times). [Yes]
           \item A description of the computing infrastructure used. (e.g., type of GPUs, internal cluster, or cloud provider). [Yes]
   \end{enumerate}

   \item If you are using existing assets (e.g., code, data, models) or curating/releasing new assets, check if you include:
   \begin{enumerate}
     \item Citations of the creator If your work uses existing assets. [Yes]
     \item The license information of the assets, if applicable. [Yes]
    \item New assets either in the supplemental material or as a URL, if applicable. [Yes]
    \item Information about consent from data providers/curators. [Yes]
     \item Discussion of sensible content if applicable, e.g., personally identifiable information or offensive content. [Not Applicable]
   \end{enumerate}

   \item If you used crowdsourcing or conducted research with human subjects, check if you include:
   \begin{enumerate}
     \item The full text of instructions given to participants and screenshots. [Not Applicable]
     \item Descriptions of potential participant risks, with links to Institutional Review Board (IRB) approvals if applicable. [Not Applicable]
     \item The estimated hourly wage paid to participants and the total amount spent on participant compensation. [Not Applicable]
   \end{enumerate}

   \end{enumerate}

\end{document}

%% file: Table1.tex
\begin{table}[t]

\centering

\small
\resizebox{\columnwidth}{!}{
\begin{tabular}{|l|c|c|}
\hline
Method & Task 1: MNLL  & Task 1: MSE   \\
\hline
Bayesian NeuralODE  & $0.82 \pm 0.01$ & $1.45 \pm 0.04$ \\
\hline
NeuralODE & --- & $0.29 \pm 0.11$ \\
\hline
npODE & $1.47 \pm 0.59$ & $0.16 \pm 0.05$  \\
\hline
GP-ODE & $0.60 \pm 0.03$ & $0.13 \pm 0.01$ \\
\hline
ODE2VAE & $0.55 \pm 0.03$ & $0.13 \pm 0.01$  \\
\hline
Latent SDE & $0.42 \pm 0.03$ & $0.10 \pm 0.01$  \\
\hline

\textbf{GP-DNF} (ours) & $\mathbf{0.12 \pm 0.01}$ & $\mathbf{0.03 \pm 0.01}$ \\
\hline
\end{tabular}
}

\vspace{5pt}

\resizebox{\columnwidth}{!}{
\begin{tabular}{|l|c|c|}
\hline
Method  & Task 2: MNLL & Task 2: MSE  \\
\hline
Bayesian NeuralODE  &  $0.88 \pm 0.01$ & $1.68 \pm 0.04$ \\
\hline
NeuralODE &  --- & $0.55 \pm 0.07$ \\
\hline
npODE &  $8.89 \pm 3.06$ & $2.08 \pm 0.78$ \\
\hline
GP-ODE &  $0.41 \pm 0.18$ & $0.21 \pm 0.07$ \\
\hline
ODE2VAE &  $0.37 \pm 0.14$ & $0.19 \pm 0.05$ \\
\hline
Latent SDE &  $0.30 \pm 0.18$ & $0.15 \pm 0.05$ \\
\hline

\textbf{GP-DNF} (ours) &  $\mathbf{0.21 \pm 0.06}$ & $\mathbf{0.04 \pm 0.01}$ \\
\hline
\end{tabular}
}
\caption{A comparison of mean negative log likelihood (MNLL) and mean squared error (MSE) performance results for regular (Task 1) and irregular (Task 2) time grids in the VDP dynamic modeling problem of Equation \ref{eq.vdp}.}\label{table1}
\end{table}




%% file: Table2.tex
\begin{table}[t]
    \centering
          \resizebox{1\width}{!}
          {\begin{tabular}{|l|c|c|}
        
        \hline
        Method  & MNLL & MSE  \\
        \hline
        Bayesian NeuralODE  & $0.77 \pm 0.12$ & $0.24 \pm 0.03$ \\
        \hline
        NeuralODE & --- & $0.18 \pm 0.00$ \\
        \hline
        npODE  & $6.49 \pm 1.49$ & $0.08 \pm 0.01$ \\
        \hline
        GP-ODE& $0.09 \pm 0.05$ & $0.07 \pm 0.02$ \\
        \hline
        ODE2VAE & $0.09 \pm 0.04$ & $0.07 \pm 0.02$ \\
        \hline
        Latent SDE & $0.07 \pm 0.03$ & $0.05 \pm 0.02$ \\
        \hline
        \textbf{GP-DNF} (ours) & $\mathbf{0.05 \pm 0.02}$ & $\mathbf{0.04 \pm 0.01}$ \\
        \hline
      \end{tabular}
      }
    
    \caption{A comparison of mean negative log likelihood (MNLL) and mean squared error (MSE) performance results for regular time grids in the FHN dynamic modeling problem of Equation \ref{eq.fnh}.}
    \label{table2}
    
\end{table}

%% file: Table3.tex
\begin{table*}[t]
\centering

\resizebox{\textwidth}{!}{
\begin{tabular}{|l|r|r|r|r|r|r|}
\hline
\textbf{Method} & \textbf{Subject 09short} & \textbf{Subject 09long} & \textbf{Subject 35short} & \textbf{Subject 35long} & \textbf{Subject 39short} & \textbf{Subject 39long} \\ \hline
Bayesian NeuralODE  & 2.03 $\pm$ 0.10 & 1.50 $\pm$ 0.05 & 1.42 $\pm$ 0.05 & 1.37 $\pm$ 0.06 & 1.61 $\pm$ 0.07 & 1.45 $\pm$ 0.03 \\ \hline
npODE & 2.09 $\pm$ 0.01 & 1.78 $\pm$ 0.08 & 1.67 $\pm$ 0.02 & 1.66 $\pm$ 0.04 & 2.06 $\pm$ 0.05 & 1.78 $\pm$ 0.04 \\ \hline
GP-ODE & 1.19 $\pm$ 0.02 & 1.14 $\pm$ 0.02 & 1.25 $\pm$ 0.06 & 1.08 $\pm$ 0.04 & 1.25 $\pm$ 0.01 & 1.36 $\pm$ 0.02 \\ \hline
ODE2VAE & 1.17 $\pm$ 0.02 & 1.12$\pm$ 0.02 & 1.21 $\pm$ 0.05 & 1.05 $\pm$ 0.04 & 1.23 $\pm$ 0.01 & 1.31 $\pm$ 0.02 \\ \hline
Latent SDE & 1.05 $\pm$ 0.02 & 1.00 $\pm$ 0.02 & 0.95 $\pm$ 0.05& 0.86 $\pm$ 0.04 & 1.08 $\pm$ 0.01 & 1.18 $\pm$ 0.02 \\ \hline
\textbf{GP-DNF} (ours) & \textbf{0.98 $\pm$ 0.02} & \textbf{0.96 $\pm$ 0.02} & \textbf{0.82 $\pm$ 0.04} & \textbf{0.79 $\pm$ 0.03} & \textbf{1.02 $\pm$ 0.01} & \textbf{1.10 $\pm$ 0.02} \\ \hline

\end{tabular}}
\caption{Mean negative log likelihood for long sequence and short sequence predictions on the held out test set.}\label{table3}

\end{table*}

%% file: Table4.tex
\begin{table*}[t]
\centering

\resizebox{\textwidth}{!}{
\begin{tabular}{|l|r|r|r|r|r|r|}
\hline
\textbf{Method} & \textbf{Subject 09short} & \textbf{Subject 09long} & \textbf{Subject 35short} & \textbf{Subject 35long} & \textbf{Subject 39short} & \textbf{Subject 39long} \\ \hline
Bayesian NeuralODE & $25.50 \pm 1.70$ & $21.32 \pm 2.58$ & $23.09 \pm 3.95$ & $20.86 \pm 2.95$ & $53.34 \pm 5.31$ & $39.66 \pm 6.82$ \\ \hline
npODE & $27.53 \pm 2.87$ & $33.83 \pm 2.46$ & $36.50 \pm 3.86$ & $23.54 \pm 0.56$ & $115.38 \pm 10.96$ & $53.51 \pm 2.98$ \\ \hline
GP-ODE & $9.11 \pm 0.37$ & $8.38 \pm 1.23$ & $10.11 \pm 0.79$ & $11.66 \pm 0.73$ & $26.72 \pm 0.63$ & $21.17 \pm 2.88$ \\ \hline
ODE2VAE & 9.05 $\pm$ 0.32 & 8.14 $\pm$ 1.05 & 9.25 $\pm$ 0.96 & 10.08 $\pm$ 1.07 & 25.25 $\pm$ 0.61 & 21.06 $\pm$ 2.14 \\ \hline
Latent SDE & 7.46 $\pm$ 0.30 & 6.45 $\pm$ 0.54 & 7.57 $\pm$ 0.46 & 7.65 $\pm$ 0.55 & 21.25 $\pm$ 0.32 & 18.72 $\pm$ 0.97 \\ \hline

\textbf{GP-DNF} (ours) & $\mathbf{7.03 \pm 0.24}$ & $\mathbf{6.04 \pm 0.45}$ & $\mathbf{6.72 \pm 0.22}$ & $\mathbf{7.03 \pm 0.23}$ & $\mathbf{19.43 \pm 0.26}$ & $\mathbf{16.21 \pm 0.73}$ \\ \hline

\end{tabular}}
\caption{ Mean squared error for the same long sequence and short sequence predictions on the held out test set.}\label{table4}

\end{table*}

%% file: mocap.tex
\begin{figure*}[ht!]
    \centering
    \includegraphics[width=\textwidth]{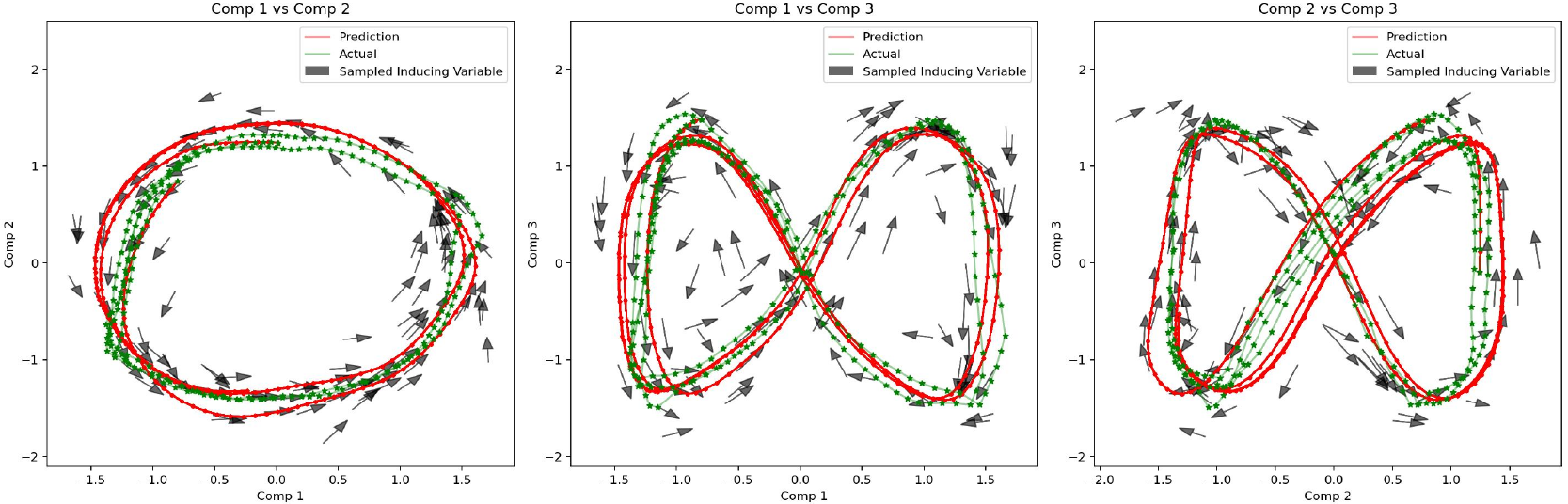}
    \caption{Three 2D projection plots from Figure \ref{3dtraj} showing vector trajectories. The size and direction of the arrows represent the projections of the sampled inducing variables. The results indicate that the model exhibits strong performance and robust generalization capabilities in time series prediction tasks.}
    \label{fig3}
\end{figure*}